\documentclass{article}

\PassOptionsToPackage{numbers, compress}{natbib}

    \usepackage[preprint]{neurips_2025}

\usepackage[utf8]{inputenc} 
\usepackage[T1]{fontenc}    
\usepackage{url}            
\usepackage{booktabs}       
\usepackage{amsfonts}       
\usepackage{nicefrac}       
\usepackage{microtype}      

\usepackage{bm}
\usepackage{tikz}
\usepackage{scalerel}
\usepackage{xspace}
\usepackage{pifont}
\usepackage{multirow}
\usepackage{tikz}
\usepackage{comment}
\usepackage{amsmath,amssymb} 
\usepackage{amsthm}
\usepackage{color}
\usepackage[inline]{enumitem}
\usepackage{float}
\usepackage{booktabs}       
\usepackage{indentfirst}
\usepackage{multirow}
\usepackage{makecell}
\usepackage{pifont}
\usepackage{everyshi}
\newcommand\doubleplus{+\kern-1.3ex+\kern0.8ex}

\usepackage{colortbl}
\usepackage{marvosym}
\usepackage[normalem]{ulem}
\usepackage{bm}
\usepackage{wrapfig}

\newtheorem{assumption}{Assumption}
\newtheorem{lemma}{Lemma}

\newtheorem{theorem}{Theorem}

\DeclareMathOperator{\vect}{vec}

\title{Random at First, Fast at Last: NTK-Guided Fourier Pre-Processing for Tabular DL}

\author{%
  Renat Sergazinov\thanks{Equal contribution, correspondence to mc.sergazinov@gmail.com} \\
  Seattle, WA 98121 \\
  \texttt{mc.sergazinov@gmail.com}
  \And
  Jing Wu\footnotemark[1] \\
  Herndon, VA 20171 \\
  \texttt{jingwu6@illinois.edu}
  \And
  Shao-An Yin\footnotemark[1]  \\
  Minneapolis, MN 55414\\
  \texttt{yin00425@umn.edu}
}

\begin{document}

\maketitle

\begin{abstract}
While random Fourier features are a classic tool in kernel methods, their utility as a pre-processing step for deep learning on tabular data has been largely overlooked. Motivated by shortcomings in tabular deep learning pipelines -- revealed through Neural Tangent Kernel (NTK) analysis -- we revisit and re-purpose random Fourier mappings as a parameter-free, architecture-agnostic transformation. By projecting each input into a fixed feature space via $\sin(\mathbf{w}^{\top}\mathbf{x})$ and $\cos(\mathbf{w}^{\top}\mathbf{x})$ with frequencies $\mathbf{w}$ sampled once at initialization, this approach circumvents the need for ad hoc normalization or additional learnable embeddings. We show within the NTK framework that this mapping (i) bounds and conditions the network’s initial NTK spectrum, and (ii) introduces a bias that shortens the optimization trajectory, thereby accelerating gradient-based training. These effects pre-condition the network with a stable kernel from the outset. Empirically, we demonstrate that deep networks trained on Fourier-transformed inputs converge more rapidly and consistently achieve strong final performance, often with fewer epochs and less hyperparameter tuning. Our findings establish random Fourier pre-processing as a theoretically motivated, plug-and-play enhancement for tabular deep learning.


\end{abstract}

\section{Introduction}

\begin{wrapfigure}[18]{r}{0.4\textwidth}
    \centering
    \includegraphics[width=0.4\textwidth]{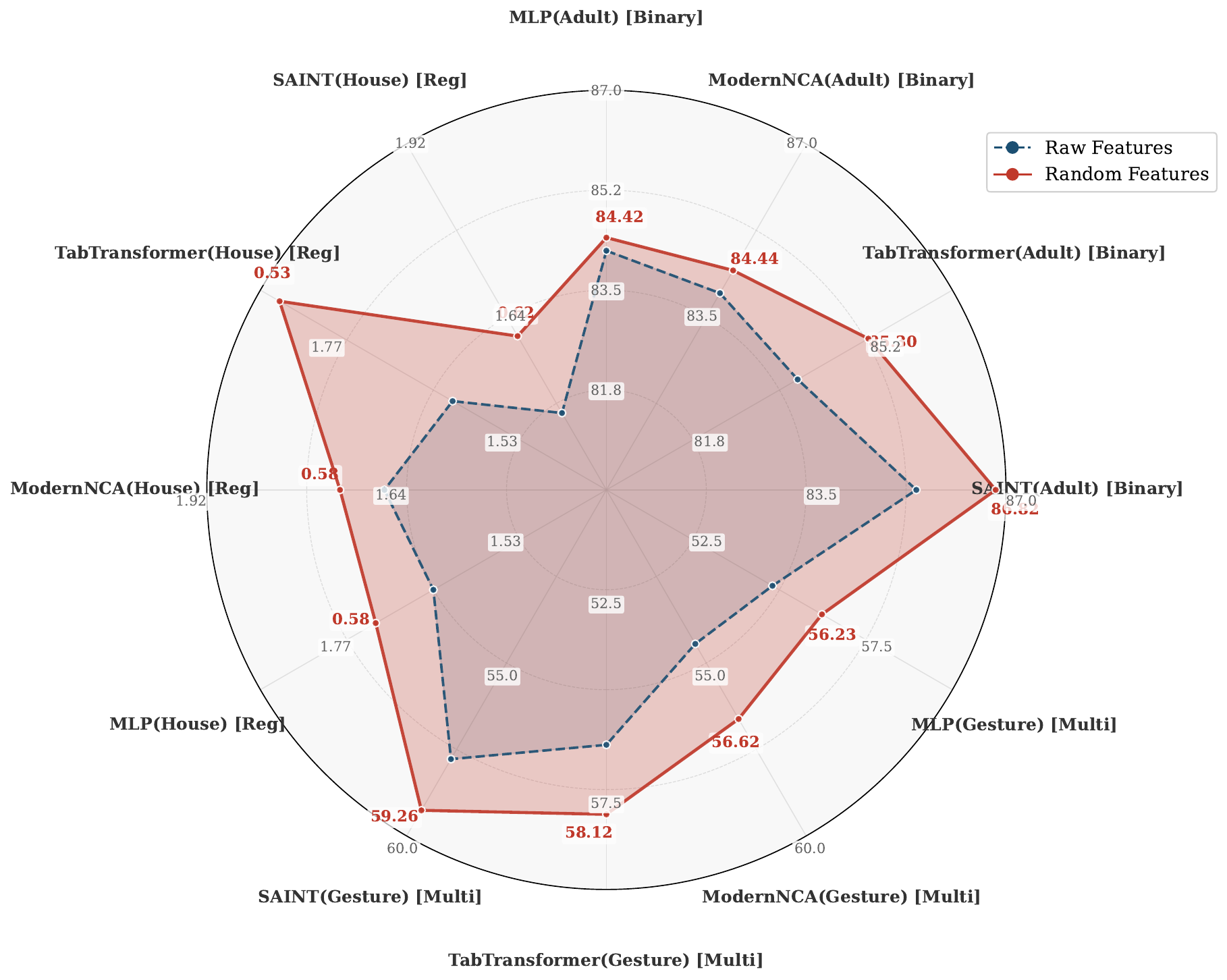}
    \vspace{-5mm}
    \caption{Performance comparison between raw and random features across four models on three datasets: Adult (binary classification), Gesture (multi-class classification), and House (regression).}
    \label{fig:radar}
    \vspace{-2mm}
\end{wrapfigure}

Deep learning (DL) has driven substantial progress across various domains including language, vision, and speech. However, tabular deep learning remains an active area of research with several recent specialized architectures such as TabNet \citep{arik2021tabnet}, TabTransformer \citep{huang_2020_tabtransformer}, SAINT \citep{somepalli2021saint}, and ModernNCA \citep{ye_2024_modernnca}.

A consistent theme in successful deep learning systems across other domains is the explicit use of domain-specific feature preprocessing steps, such as convolutions in vision \citep{lecun1989handwritten}, byte-pair encoding in NLP \citep{gage1994new}, and STFT/MFCC pipelines in speech processing. In contrast, most tabular DL methods primarily rely on simple normalization layers without a theoretically justified preprocessing stage. Recent works like feature tokenizer \citep{gorishniy2022embeddings} and RealMLP \citep{holzmuller_2025_realmlp} have begun to explore learned preprocessing approaches; however, the lack of theoretical guidance makes these methods challenging to tune effectively in practice.

We revisit tabular DL through the lens of neural network kernel theory. Specifically, we view finite-width neural networks trained via gradient descent as kernel machines characterized by their neural network tangent kernels (NTKs) as proposed in \citet{jacot_2018_ntk,adlam_2020_highdimntk}. \citet{jeffares_2025_telescopingntk} recently highlighted a notable parallel between neural networks and gradient-boosted trees (GBTs) through their kernel formulations, pinpointing two critical issues with neural network kernels: (1) \textit{unboundedness}, causing instability and unpredictable performance on test samples, especially under distribution shifts, and (2) \textit{slow adaptation}, as initial kernels are data-agnostic and adapt only slowly through training.

To address these issues, we propose a simple, parameter-free pre-processing step using random Fourier features (RFF):
\begin{equation}
\mathbf{x} \mapsto \left[\sin\bigl(\mathbf{w}^\top \mathbf{x}\bigr), \cos\bigl(\mathbf{w}^\top \mathbf{x}\bigr)\right], \quad \mathbf{w} \sim p(\mathbf{w})\,.
\end{equation}
Random Fourier features approximate a shift-invariant kernel determined by the distribution $p(\mathbf{w})$ \citep{rahimi2007random}, naturally ensuring boundedness and stability of the kernel spectrum. Our theoretical analysis confirms that this pre-processing step:
\begin{itemize}
\item \textbf{Stabilizes the kernel:} Ensuring bounded kernel values and mitigating instability due to extreme gradient magnitudes.
\item \textbf{Introduces a beneficial inductive bias:} Facilitating faster adaptation of the model during training, thus accelerating convergence.
\end{itemize}

Empirically, our "plug-and-play" random Fourier pre-processing leads to faster training convergence and achieves comparable or superior accuracy on a broad set of tabular benchmarks, without introducing additional parameters and with minimal runtime overhead.

In summary, our contributions are:
\begin{itemize}
\item \textbf{Method:} A simple and effective random Fourier preprocessing step applicable to various tabular DL architectures.
\item \textbf{Theory:} Kernel-based insights showing how random Fourier features effectively address kernel instability and slow adaptation identified by \citet{jeffares_2025_telescopingntk}.
\item \textbf{Practice:} Extensive empirical validation demonstrating improved convergence speed and robust performance on real-world datasets (Figure~\ref{fig:radar}).
\end{itemize}
These results position random Fourier features as a theoretically grounded and practical preprocessing solution for enhancing the stability and efficiency of tabular deep learning models.

\section{Related Works}
\label{sec:related}

Research on deep learning for tabular data has evolved along two largely independent axes: the \emph{choice of backbone architecture} and the \emph{design of feature-processing schemes}. Our work primarily advances the latter through principled analysis based on neural tangent kernels (NTK). We review each strand and position our contribution within the NTK framework.

\subsection{Architectural backbones}
Existing backbones for tabular deep learning broadly belong to three families: custom DNN architectures, Transformer-based models, and retrieval-augmented networks.

\emph{Custom DNN architectures.} This research strand investigates various modifications to traditional neural architectures by introducing custom layers, novel activation functions, altering weight-matrix parameterizations, and employing specialized loss functions. Representative examples include RealMLP \citep{holzmuller_2025_realmlp}, which introduces scalar expansions, Abstract-DNN \citep{sotoudeh_2020_ann}, which utilizes abstract interpretation for robustness, Self-Normalizing Networks (SNN) \citep{klambauer_2017_snn}, designed to stabilize training through normalization, Kolmogorov–Arnold Networks (KAN) \citep{liu_2024_kan}, which employ adaptive functional forms inspired by the Kolmogorov–Arnold representation theorem, SAINT \citep{somepalli2021saint} integrating self-attention with differentiable embedding layers, NODE \citep{popov2019neural} implementing differentiable decision trees, TabNet \citep{arik2021tabnet} introducing sequential attention mechanisms, and tree-inspired neural architectures \citep{hazimeh_2020_treednn}, which leverage hierarchical decision processes.

\emph{Transformer-based models.} Inspired by the transformative success of attention mechanisms in sequential data modeling \citep{vaswani_2017_attention}, TabTransformer \citep{huang_2020_tabtransformer} and TabLLM \citep{hegselmann_2023_tabllm} adapt transformers to the heterogeneous structure of tabular data, introducing customized positional encodings and attention frameworks explicitly tailored for non-sequential inputs.

\emph{Retrieval-augmented networks.} These approaches integrate non-parametric memory components inspired by retrieval systems, leveraging nearest-neighbor retrieval strategies. Notable examples include ModernNCA \citep{ye_2024_modernnca} and TabR \citep{gorishniy_2024_tabr}, which combine parametric encoders with external memory banks, enhancing interpretability, controllability, and efficiency.

\subsection{Feature processing for tabular DNNs}
Unlike in vision or natural language processing—where convolutional filters, tokenization, and spectrogram transformations are standard practice—tabular deep learning still lacks a universally accepted, theoretically principled preprocessing standard.

\emph{Self-supervised embeddings.} Several methods propose learning additional feature embeddings via self-supervised pre-training objectives. For instance, SwitchTab \citep{wu_2024_switchtab}, PTaRL \citep{ye_2024_ptarl}, VIME \citep{yoon2020vime}, TaBERT \citep{yin2020tabert}, ReconTab \citep{chen2023recontab}, and SCARF \citep{bahri2021scarf} all leverage self-supervised learning strategies to enhance representation quality. These methods often achieve significant performance boosts (e.g., SwitchTab reports improvements up to +5 percentage points). However, these techniques introduce additional complexity, including a secondary training stage and a substantial hyper-parameter search space.

\emph{Learned deterministic transforms.} Another class of feature processing methods employs learned but deterministic transformations that are jointly optimized with the primary model. For example, Feature Tokenization \citep{gorishniy2022embeddings} and RealMLP’s scalar expansions \citep{holzmuller_2025_realmlp} provide simple yet interpretable transformations on each raw input variable, yielding reliable accuracy gains at moderate computational cost. Nonetheless, these approaches still lack explicit theoretical insights, complicating hyper-parameter tuning and optimization.

\subsection{DNNs through the kernel lens}
The Neural Tangent Kernel (NTK) theory \citep{jacot_2018_ntk} shows that infinitely wide neural networks trained via gradient descent can be precisely described as kernel machines, governed by the NTK. Subsequent research has analyzed NTK properties such as approximation power to standard-width DNNs, convergence behavior, and provided explanations for phenomena like double-descent and grokking \citep{adlam_2020_highdimntk,chen_2020_generalizedntk,domingos_2020_sgdntk,atanasov_2021_alignmentntk,jeffares_2025_telescopingntk}. See \citet{golikov_2022_ntk_survey} for a comprehensive review. Two critical insights from this line of research particularly relevant to tabular data are:

\begin{itemize}
\item \textbf{Unbounded kernels.} The conventional MLP-induced NTK can produce unbounded kernel values, potentially degrading generalization under distribution shifts commonly observed in tabular datasets.
\item \textbf{Delayed beneficial optimization phases.} NTK-driven insights reveal beneficial phenomena such as "phase shifts" (e.g., grokking) that typically occur late in training, conflicting with the practical requirement for quick convergence in tabular scenarios.
\end{itemize}

\paragraph{Our contribution}
We systematically address the feature processing challenge for tabular deep learning by proposing a theoretically principled preprocessing step. Specifically, we deterministically map input features using fixed random Fourier features \citep{rahimi2007random}. This preprocessing strategy (i) guarantees bounded, well-conditioned kernels before training commences and (ii) introduces a beneficial inductive bias, thus significantly shortening the gradient descent optimization path. The method is parameter-free, applicable to any architecture, and significantly reduces feature-specific hyper-parameter tuning. Our comprehensive experiments empirically validate these theoretical claims, demonstrating consistent improvements in training convergence and final accuracy compared to competitive tabular DL baselines.

\section{Methodology}
\label{sec:method}

Our goal is to equip any deep network for tabular data with a \emph{parameter‑free}, plug‑and‑play feature transform that (i) requires no extra training, (ii) improves performance, and (iii) admits a clear theoretical rationale.  We achieve this by deterministically mapping each raw input~\(\mathbf{x}\in\mathbb{R}^d\) into a fixed \emph{random Fourier feature} space before feeding it to the network.

\subsection{Random Fourier Feature Transform}
\label{subsec:rff}

\begin{figure}[t]
    \centering
    \begin{minipage}{0.55\linewidth}
        \centering
        \includegraphics[width=\linewidth]{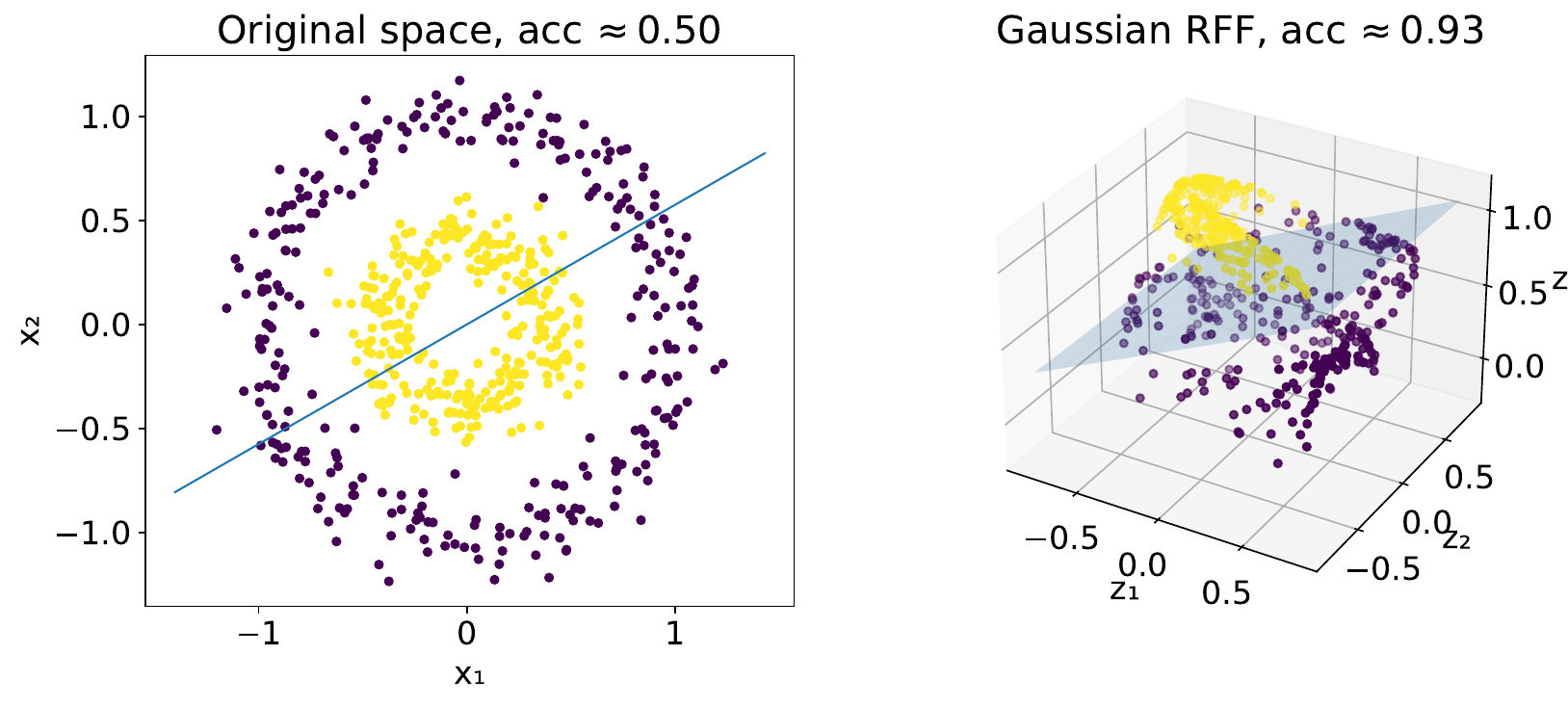}
    \end{minipage}
    \hfill
    \begin{minipage}{0.4\linewidth}
        \centering
        \resizebox{\columnwidth}{!}{%
        \begin{tabular}{@{}lll@{}}
            \toprule
            \textbf{Kernel Name} & \( k(\Delta) \) & \( p(\omega) \) \\
            \midrule
            Gaussian & \( e^{-\frac{\|\Delta\|_2^2}{2}} \) & \( (2\pi)^{-\frac{D}{2}} e^{-\frac{\|\omega\|_2^2}{2}} \) \\
            Laplacian & \( e^{-\|\Delta\|_1} \) & \( \prod_d \frac{1}{\pi (1 + \omega_d^2)} \) \\
            Cauchy & \( \prod_d \frac{2}{1 + \Delta_d^2} \) & \( e^{-\|\Delta\|_1} \) \\
            \bottomrule
        \end{tabular}
        }
    \end{minipage}
    \caption{Random Fourier Features: (left) visualization of the mapping process from input space to feature space, and (right) a list of popular shift-invariant kernels along with their Fourier transforms.}
    \label{fig:rff-combined}
\end{figure}

Let \(p(\omega)\) be the Fourier transform of a chosen shift‑invariant kernel \(k(\mathbf{x},\mathbf{x}^\prime) = k(\bm{\Delta})\).  We draw once at initialization
\[
\Omega = [\omega_1,\dots,\omega_D]\in\mathbb{R}^{d\times D}, 
\quad \omega_i\sim p(\omega),
\]
and define
\[
\Phi(\mathbf{x}) = \sqrt{\tfrac{2}{D}}\bigl[\,
\sin(\Omega^\top \mathbf{x})\,,\cos(\Omega^\top \mathbf{x})
\bigr]\in\mathbb{R}^{2D},
\]
where $\sin, \cos$ are applied element-wise. Crucially, \(\Omega\) remains fixed throughout training and inference.  Any architecture—MLP, transformer, or custom block—can simply replace its original input \(\mathbf{x}\) with \(\Phi(\mathbf{x})\), or concatenate them as needed.  In practice, one selects \(p(\omega)\) to match desired kernel properties (e.g.\ Gaussian for RBF) and picks \(D\) large enough to capture the dominant spectral mass of \(p\).

In Figure~\ref{fig:rff-combined}, we illustrate how transforming features using random Fourier features approximates a kernel machine based on the chosen sampling distribution \citep{rahimi2007random}. In this case, the data clearly follows a Gaussian kernel. Trying to fit a logistic regression in the original space leads to a random guess (subpar) accuracy of prediction. Transforming the data by sampling frequency from the 2-dimensional Gaussian distribution, we select $D = 3$ for the number of Monte Carlo samples. For illustration purposes, we also use an equivalent formulation of the random Fourier features $\tilde \Phi(\mathbf{x})  = D^{-1/2} \cos(\Omega^T \mathbf{x} + \mathbf{b})$, where $\mathbf{b}$ is sampled from the $\mathcal{U}(0,2\pi)$ -- see the note from \citet{sutherland2015error} on when $\Phi$ is preferred over $\tilde \Phi$. We plot the resulting transformed features on the right. Fitting a logistic regression, we now obtain a significantly higher accuracy.

\subsection{Optimization as Weight‑Averaging and Kernel Representation}
\label{sec:approx}

Consider an $m$-layer Multi-layer Perceptron (MLP), define the output of the $(k+1)$-th hidden layer as
\begin{equation}
    \mathbf{h}^{k+1}=\sigma_k(W_k \mathbf{h}^{k}) \in \mathbb{R}^{d_{k+1}},
\end{equation}
with the initial input $\mathbf{h}^{0}=\mathbf{x} \in \mathbb{R}^d$. The final output is given by
\begin{equation}\label{equ:mlp}
    \hat{y} = f_\theta(x)=\mathbf{h}^{m}\in\mathbb R,
\end{equation}
with the weight matrices $W_k$ collected in $\theta=\{W_k\}_{k=0}^{m-1}$.

Following \citet{jeffares_2025_telescopingntk}, we build on the telescoping model approximation to the trained DNN, $f_{\theta_T}$ (defined in equation (\ref{equ:mlp})), which expresses the network output after \(T\) SGD steps as
\begin{equation}
f_{\boldsymbol{\theta}_T}(\mathbf{x})=
f_{\boldsymbol{\theta}_0}(\mathbf{x})
- \sum_{t=1}^{T}\sum_{i\in [n]}
K_t(\mathbf{x},\mathbf{x}_i)\,g_{it}^\ell.
\label{eq:telescoping}
\end{equation}
Here \(g_{it}^\ell\) is the per‑example loss gradient and
\(
K_t(\mathbf{x},\mathbf{x}_i)
=\nabla_{\boldsymbol{\theta}}f_{\boldsymbol{\theta}_{t-1}}(\mathbf{x})^\top
\nabla_{\boldsymbol{\theta}}f_{\boldsymbol{\theta}_{t-1}}(\mathbf{x}_i)
\)
is the inner‑product “kernel’’ at step \(t\).  In practice, this kernel is
\emph{randomly initialized} and can be \emph{unbounded} due to exploding or
vanishing gradients, especially on tabular features with heavy‑tailed or
heterogeneous distributions.  We hypothesize that this unpredictability and poor conditioning hinder fast, reliable convergence of a DNN. 

\subsection{Fourier‑Feature Preprocessing and Kernel Regularization}
To tame the unbounded, random kernel in \ref{eq:telescoping}, we pre-process inputs via
fixed random Fourier features \(\Phi(\mathbf{x})\) (see \S\ref{sec:method}).
Writing the SGD approximation for the transformed network
\(\tilde f_{\theta_T}\circ \Phi\) yields an analogous kernel term
\[
\tilde K_t(\mathbf{x},\mathbf{x}_i)
=
\nabla_{\boldsymbol{\theta}}f_{\boldsymbol{\theta}_{t-1}}\bigl(\Phi(\mathbf{x})\bigr)^\top
\nabla_{\boldsymbol{\theta}}f_{\boldsymbol{\theta}_{t-1}}\bigl(\Phi(\mathbf{x}_i)\bigr).
\]
Our main theoretical result shows that \(\{\tilde K_t\}\) become uniformly
bounded, and that the Fourier map injects an explicit bias term tied to
the sampling distribution.


\begin{assumption}
\label{ass:main_assumption}
We make the following standard assumptions:
\begin{enumerate}
    \item The Multi-Layer Perceptron $f_{\bm{\theta}}(\cdot)$ is non-degenerate, i.e., $f_{\bm{\theta}}(\mathbf{x}) \neq \text{const}$ for all $\mathbf{x} \in \mathbb{R}^d$.
    \item Each activation function $\sigma_k$ in every layer $k$ of the neural network is $1$‑Lipschitz and $s$‑homogeneous.
\end{enumerate}
\end{assumption}

\begin{theorem}[Kernel Boundedness and Bias]
\label{thm:boundedness}
Consider kernels $K_t$ and $\tilde K_t$ associated with a neural network $f_\theta: \mathbf{x} \to \hat{y}$, trained without and with random Fourier features $\Phi(\mathbf{x})$, respectively. Then, the following properties hold:
\begin{enumerate}
    \item As the input feature becomes unbounded, i.e., $\| \mathbf{x}\| \to \infty$, the kernel $K_t$ becomes unbounded, specifically,
    \[
    \sup_{\mathbf{x},\mathbf{x}'} K_t(\mathbf{x},\mathbf{x}') = \infty.
    \]
    \item Under Assumption \ref{ass:main_assumption}, the kernel $\tilde K_t$ admits an upper-bound decomposition:
    \[
    \tilde K_t(\mathbf{x}, \mathbf{x}_i) = \gamma_t(\mathbf{x}, \mathbf{x}_i) + m_t(\mathbf{x}, \mathbf{x}_i),
    \]
    where $\gamma_t$ is a bounded, positive-definite kernel, and $m_t$ is a bounded function approximating a shift-invariant kernel determined by the sampling distribution $p(\omega)$. Specifically,
    \[
    \sup_{\mathbf{x},\mathbf{x}'} \gamma_t(\mathbf{x},\mathbf{x}') < \infty, \quad \text{and} \quad \sup_{\mathbf{x},\mathbf{x}'} m_t(\mathbf{x},\mathbf{x}') < \infty.
    \]
\end{enumerate}
\end{theorem}

\begin{proof}
We provide full proof in Appendix~\ref{app:proof-thm1}.
Here we provide intuition for the main idea. 

\paragraph{Step 1:  Why the plain NTK is unbounded.}
The key observation is that there exists a direction in the input space of the MLP satisfying Assumption~\ref{ass:main_assumption} such that the gradient grows at least linearly (for $t$-homogeneous activation) with the input norm. That is we have:
$$
\|\nabla_\theta f_\theta(\alpha \cdot \mathbf{x})\| \ge C_2 \, \alpha^{s^{m-1}},
\qquad\forall\;\alpha>0 .
$$

Because the neural–tangent kernel is the squared $2$‑norm of this same gradient,
$K_t(\alpha \mathbf{x}, \alpha \mathbf{x}) = \left\| \nabla_{\boldsymbol{\theta}}f_{\boldsymbol{\theta}_{t-1}}(\alpha \mathbf{x}) \right\|_2^2 \ge C_2^2 \alpha^{2s^{m-1}}$,
it grows with $\alpha$. Hence, we have
$\sup_{\mathbf{x},\mathbf{x}'}K_t(\mathbf{x},\mathbf{x}')=\infty .$

\paragraph{Step 2:  How Fourier features tame the explosion.}
Insert a random Fourier map $\Phi(\mathbf{x})= \sqrt{\tfrac{2}{D}}\bigl[\,
\sin(\Omega^\top \mathbf{x})\,,\cos(\Omega^\top \mathbf{x})
\bigr]$ as the new first “layer”.
Two facts change the picture:

\begin{enumerate}
\item \textbf{Input‐norm control.}
$\|\Phi(\mathbf{x})\|\le\sqrt2$ independently of $\|\mathbf{x}\|$.
Consequently, we can uniformly bound the gradient norm.

\item \textbf{Kernel decomposition.}
The full gradient splits into

$$
    \nabla_{\boldsymbol{\theta}} f_{\boldsymbol{\theta}_{t-1}}(\Phi(\mathbf{x})) = \left [ \nabla_{\bm{\psi}} g_{\bm{\psi}_{t-1}}(\mathbf{h}_1),  \frac{\partial f}{\partial \mathbf{h}^1_1} \Phi(\mathbf{x}), \dots, \frac{\partial f}{\partial \mathbf{h}^1_{d_1}} \Phi(\mathbf{x})\right],
$$
where $g$ is the truncated MLP consisting of layers $1, 2, \dots, m$. The first block involves only upper‑layer weights $\psi$ and remains bounded by point (1) above. The second block is linear in the Fourier features.  Using results of \citet{rahimi2007random} one shows
$\mathbb{E}_{\omega}[\Phi(\mathbf{x})^\top\Phi(\mathbf{x}')]=k(\mathbf{x}-\mathbf{x}')$,
so this part of the NTK behaves like a \emph{bounded} shift‑invariant kernel.
\end{enumerate}

Collecting the two blocks gives the announced decomposition

$$
\tilde K_t(\mathbf{x},\mathbf{x}') \;=\;
\underbrace{\gamma_t(\mathbf{x},\mathbf{x}')}_{\text{bounded}}
+\;
\underbrace{m_t(\mathbf{x},\mathbf{x}')}_{\text{shift‑inv.\ kernel}},
$$

with both summands uniformly bounded in $(\mathbf{x},\mathbf{x}')$.  This completes the sketch.

\end{proof}

\subsection{Implications of Theorem~\ref{thm:boundedness}}
\paragraph{Controlled gradients and stable training.}
A bounded kernel ensures that parameter gradients remain finite, even when encountering out-of-distribution or large-magnitude inputs. Practically, this prevents exploding gradient updates, enabling higher learning rates and fostering more stable and predictable optimization dynamics. This directly contributes to enhanced training stability and faster convergence in practice.

\paragraph{Regularisation through random features.}
Random Fourier features effectively function as an architectural regularizer by limiting the Lipschitz constant of initial layers. Simultaneously, they provide an unbiased approximation of a stable, shift-invariant kernel. Consequently, the network benefits from kernel-based smoothness properties without sacrificing depth, flexibility, or expressivity.

\paragraph{Design guidance for large-scale models.}
Our theoretical analysis offers practical design insights for developing large-scale neural network models. Specifically, it suggests incorporating bounded, norm-preserving feature mappings such as random Fourier features, positional encodings, or orthogonal projections when dealing with high-dimensional or high-magnitude inputs. This strategy mitigates kernel instability commonly encountered in deep, homogeneous architectures and enhances model robustness against distributional shifts.

\paragraph{Perspective on kernel–network hybrids.}
The kernel decomposition $\tilde K_t = \gamma_t + m_t$ reveals that neural networks with a random Fourier preprocessing step implicitly ensemble two kernels: (i) a learned, bounded kernel $\gamma_t$, which captures complex hierarchical feature relationships adaptively, and (ii) a fixed, shift-invariant kernel $m_t$, anchored directly to the input data structure. This hybrid interpretation elucidates why these networks can simultaneously exhibit rapid adaptation (through $\gamma_t$) and robust extrapolation capabilities (through $m_t$).

Theorem \ref{thm:boundedness} thus ensures that, after random Fourier pre-processing, the augmented kernel $\tilde K_t$ remains well-conditioned (avoiding extreme eigenvalues) and introduces a beneficial additive bias term $m_t$. This bias effectively pre-conditions the neural network towards the target function, simplifying the optimization path. Applying this insight to Eq.~\ref{eq:telescoping}, we obtain:

$$
\tilde f_{\boldsymbol{\theta}_T}(\mathbf{x}) = f_{\boldsymbol{\theta}_0}(\mathbf{x}) - \sum_{t=1}^T \sum_{i \in [n]} \bigl[\gamma_t(\mathbf{x},\mathbf{x}_i)+m_t(\mathbf{x},\mathbf{x}_i)\bigr] g_{it}^\ell,
$$

where the $\gamma_t$ term models adaptive, data-driven similarity, and the $m_t$ term expedites early training by providing a global offset. Together, these components explain why the inclusion of random Fourier features stabilizes the training process and significantly shortens the optimization trajectory, resulting in faster convergence without the need for additional learnable parameters.

\begin{figure}[t]
    \centering
    \includegraphics[width=\linewidth]{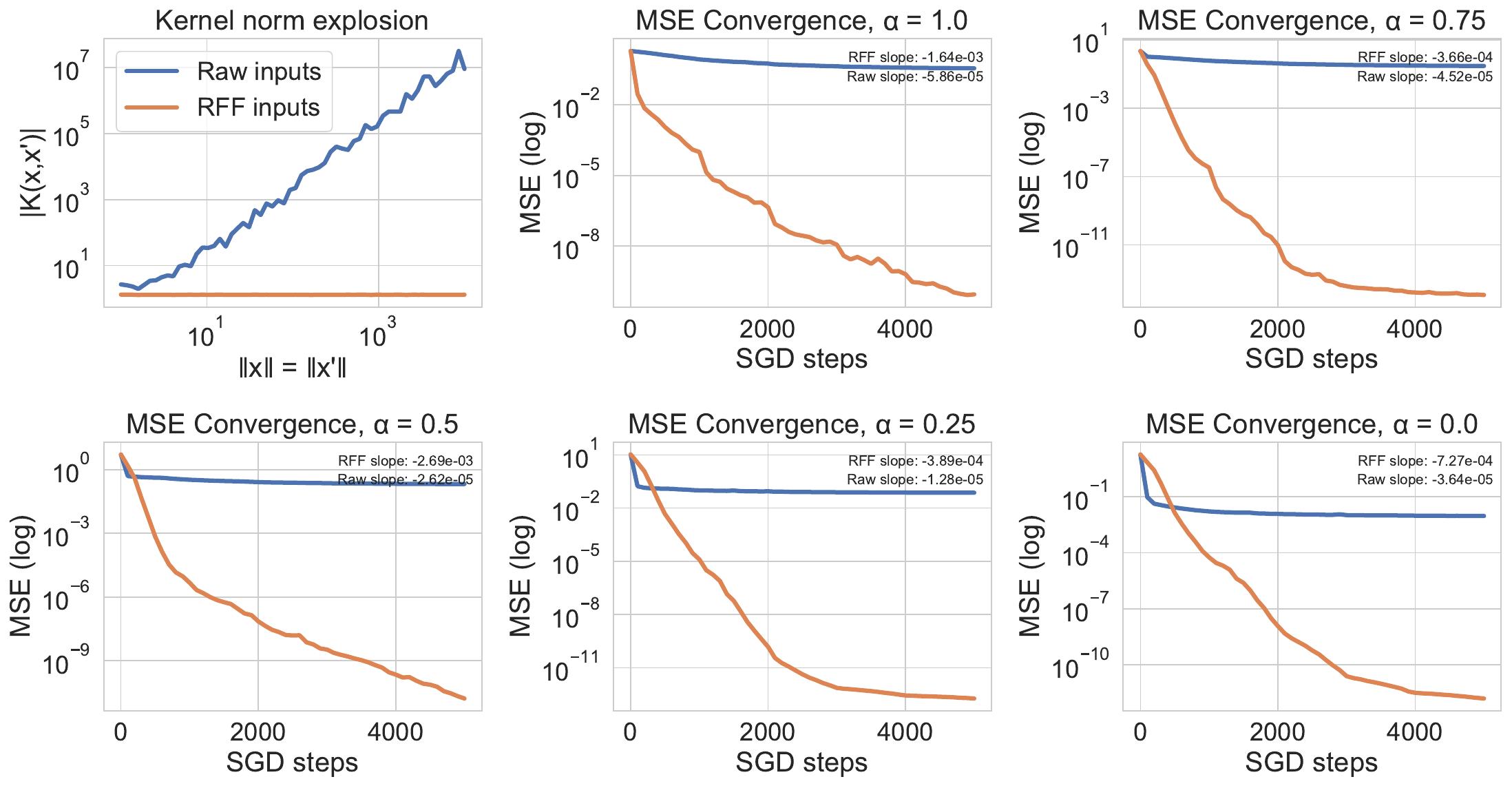}
    \caption{\textbf{Kernel norm explosion (first plot):} neural network kernel norm with and without the random Fourier features across different scales of the input features. \textbf{Neural network convergence:} convergence of the MSE of the neural network with and without the random Fourier features.}
    \label{fig:simul}
\end{figure}

\section{Simulations}

To illustrate Theorem~\ref{thm:boundedness}, we conducted two controlled simulation experiments detailed below.

\subsection{Synthetic Data with Adjustable Kernel Bias}

We generated synthetic data as $\mathbf{x} \sim \mathcal{N}(\mathbf{0}, I_d)$ with dimension $d=10$. Employing a fixed random Fourier feature mapping $\Phi: \mathbb{R}^d \to \mathbb{R}^{2D}$ (with $D=512$ and bandwidth $\sigma=1$), we constructed regression targets using a convex combination of raw linear and random Fourier features components:
\begin{equation}
y = (1 - \alpha)\langle w_{raw}, \mathbf{x} \rangle + \alpha \langle w_{rff}, \Phi(\mathbf{x}) \rangle + \varepsilon,
\end{equation}
where $\varepsilon \sim \mathcal{N}(0, 0.1^2)$ and $\alpha \in [0,1]$ controls kernel alignment. Specifically, $\alpha=1$ represents complete alignment with the kernel structure, $\alpha=0$ indicates no kernel structure, and intermediate values interpolate between these extremes.

\subsection{Experiment 1: Kernel Norm Explosion}

We initialized two identical ReLU multi-layer perceptrons (MLPs), each with two hidden layers of 128 units. The only distinction was in input processing: one received raw inputs directly, while the other utilized fixed random Fourier features pre-processing. We systematically varied the input radius $\|\mathbf{x}\| = \|\mathbf{x}'\|$ from $1$ to $10^{4}$ and calculated the neural tangent kernel (NTK) values.

\paragraph{Observation 1 (Kernel Boundedness).} The NTK of the raw-input MLP exhibited polynomial growth with increasing input norms, confirming the unbounded behavior described by Theorem~\ref{thm:boundedness}. Conversely, the NTK with random Fourier features pre-processing stabilized at finite values, consistent with the bounded decomposition $\tilde{K}_t = \gamma_t + m_t$.

\subsection{Experiment 2: Convergence Dynamics}

For kernel alignment levels $\alpha \in \{0, 0.5, 1\}$, we trained two model variants on $n=4,000$ samples:
\begin{itemize}
\item \textbf{Vanilla MLP}: Using raw inputs directly.
\item \textbf{RFF-MLP}: Identical architecture with fixed random Fourier features pre-processing.
\end{itemize}

We monitored the full-batch mean-squared error (MSE) at intervals of 100 iterations over 5,000 SGD steps. Figure~\ref{fig:simul} presents the learning curves, highlighting initial error bias (difference in MSE at iteration 0) and convergence rate (slope from linear regression of $\log(\text{MSE})$ against iterations).

\paragraph{Observation 2 (Bias Advantage).} When kernel structure was present ($\alpha > 0$), the MLP wtih random Fourier features consistently exhibited lower initial MSE at the onset of training.

\paragraph{Observation 3 (Stability Advantage).} Independent of $\alpha$, models using random Fourier features pre-processing consistently showed faster error reduction, demonstrating improved stability and less susceptibility to gradient explosions.

\begin{table*}[t]
\centering
\small
\resizebox{\textwidth}{!}{
\begin{tabular}{l|cccccc|cc}
\toprule
& \multicolumn{6}{c|}{\textbf{Classification Tasks (Accuracy $\uparrow$)}} & \multicolumn{2}{c}{\textbf{Regression Tasks (RMSE $\downarrow$)}} \\
\midrule
\textbf{Dataset} & Adult & Churn & Covertype & Gesture & Higgs & Santander & California & House \\
\textbf{Dataset size} & 48,842 & 7,043 & 495,141 & 9,873 & 98,050 & 709,877 & 20,640 & 1,000,000 \\
\textbf{Feature size} & 14 & 20 & 76,020 & 32 & 28 & 369 & 8 & 28 \\
\midrule
MLP (Raw) & 84.19 $\pm$ 0.06 & 85.28 $\pm$ 0.06 & 87.60 $\pm$ 0.04 & 54.80 $\pm$ 0.10 & 70.64 $\pm$ 0.07 & 91.15 $\pm$ 0.03 & 0.46 $\pm$ 0.00 & 0.61 $\pm$ 0.00 \\
\rowcolor{gray!20} MLP (Random) & 84.42 $\pm$ 0.06 & 86.48  $\pm$ 0.07 & 88.04 $\pm$ 0.12 & 56.23 $\pm$ 0.10 & 71.25 $\pm$ 0.07 & 91.83 $\pm$ 0.02 & 0.42 $\pm$ 0.00 & 0.58 $\pm$ 0.00 \\
\midrule
ModernNCA (Raw) & 83.98 $\pm$ 0.03 & 85.60 $\pm$ 0.08 & 91.71 $\pm$ 0.03 & 54.45 $\pm$ 0.25 & 71.90 $\pm$ 0.05 & 91.34 $\pm$ 0.01 & 0.42 $\pm$ 0.00 & 0.60 $\pm$ 0.00 \\
\rowcolor{gray!20} ModernNCA (Random) & 84.44 $\pm$ 0.05 & 86.34 $\pm$ 0.05 & 92.82 $\pm$ 0.04 & 56.62 $\pm$ 0.12 & 72.28 $\pm$ 0.06 & 92.10$\pm$ 0.02 & 0.40 $\pm$ 0.00 & 0.58 $\pm$ 0.00  \\
\midrule
TabTransformer (Raw) & 83.87 $\pm$ 0.05 & 83.58 $\pm$ 0.17 & 91.92 $\pm$ 0.09 & 56.38 $\pm$ 0.28 & 72.07 $\pm$ 0.11 & 91.02 $\pm$ 0.02 & 0.52 $\pm$ 0.00 & 0.62 $\pm$ 0.00 \\
\rowcolor{gray!20} TabTransformer (Random) & 85.30 $\pm$ 0.09 & 84.61 $\pm$ 0.20 & 93.33 $\pm$ 0.06 & 58.12 $\pm$ 0.26 & 72.22 $\pm$ 0.07 & 92.73$\pm$ 0.02 & 0.44 $\pm$ 0.01 & 0.53 $\pm$ 0.00 \\
\midrule
SAINT (Raw) & 85.43 $\pm$ 0.16 & 84.37 $\pm$ 0.28 & 99.13 $\pm$ 0.01 & 57.78 $\pm$ 0.25 & 72.38 $\pm$ 0.18 & 91.70 $\pm$ 0.05 & 0.46 $\pm$ 0.00 & 0.66 $\pm$ 0.00 \\
\rowcolor{gray!20} SAINT (Random) & 86.82 $\pm$ 0.14 & 85.12 $\pm$ 0.36 & 99.49 $\pm$ 0.01 & 59.26 $\pm$ 0.32 & 73.42 $\pm$ 0.16 & 93.02 $\pm$ 0.06 & 0.40 $\pm$ 0.00 & 0.62 $\pm$ 0.00 \\
\bottomrule
\end{tabular}}
\caption{Comparison of raw vs.\ random features across classification and regression tasks. Results are averaged over three trials. For classification tasks, higher accuracy ($\uparrow$) is better; for regression tasks, lower RMSE ($\downarrow$) is better.}
\label{tab:raw-random}
\end{table*}

\section{Experiments}
\label{sec: exp}
In this section, we present the results of our comprehensive experiments on multiple datasets to evaluate the effectiveness of random features. The section is structured in two parts. First, we provide preliminary details on the datasets, data preprocessing, and training procedures to ensure transparency and reproducibility. Second, we assess the performance of our proposed method against state-of-the-art deep learning and traditional models using established and widely used benchmark datasets \cite{gorishniy2021revisiting}.

\subsection{Preliminaries for Experiments}

\subsubsection{Datasets.}
We evaluate the performance of the proposed method on a standard benchmark introduced by \citet{gorishniy2021revisiting}. Specifically, the datasets include: California Housing \cite{pace1997sparse}, Adult \cite{kohavi1996scaling}, Churn \cite{ibm2019telco}, Higgs \cite{baldi2014searching}, Covertype \cite{blackard1999comparative}, Microsoft (MI) \cite{qin2013introducing}, Gestures \cite{openmlGesture4538}, Santander \cite{mondal_santander_2019}, California Housing (Inria) \cite{inria_california_housing}, and House \cite{openml_house_16h_574}. We report each dataset's size, number of features, and associated task in Table~\ref{tab:raw-random}.

\subsubsection{Preprocessing of Datasets}
The preprocessing pipeline addresses challenges in heterogeneous tabular data through systematic handling of missing values and feature-specific transformations. Numerical features are standardized to zero mean and unit variance, while categorical features are encoded using model-specific techniques: label encoding for transformer-based architectures such as SAINT and TabTransformer, and one-hot encoding for traditional models. To manage dimensionality, categorical features with more than n/10 unique values are excluded. Additionally, we apply stratified train-validation-test partitioning to preserve distributional characteristics across splits, ensuring reliable model evaluation on tabular benchmarks.

\subsubsection{Training Details}
We employ Adam optimization with $L^2$ regularization and task-specific loss functions: cross-entropy for classification and mean squared error (MSE) for regression. For SAINT, we specifically optimize transformer depth and attention heads while using a predefined attention mechanism (attention=`colrow'). Hyperparameter optimization is conducted using Optuna's Bayesian search with predefined seeds for reproducibility, ensuring systematic exploration of architectural configurations while maintaining consistent evaluation protocols across experimental runs.

\subsection{Results}
Table~\ref{tab:raw-random} illustrates the consistent performance advantage provided by random feature initialization over raw features across four distinct architectures and eight benchmark datasets. In classification tasks, the average accuracy improvements range between 0.79 and 1.25 percentage points (pp), highlighted by specific gains such as +1.43 pp for TabTransformer on Adult, +1.11 pp for ModernNCA on Covertype, and +1.74 pp for TabTransformer on Gesture. In regression tasks, the root mean squared error (RMSE) was reduced by an average of 11.7\%, with notable examples including a 13.0\% RMSE reduction for SAINT on California Housing and a 14.5\% reduction for TabTransformer on House. The single largest improvement observed was a +1.39 pp accuracy gain by SAINT on the Adult dataset, underscoring the robustness and effectiveness of the proposed random Fourier feature initialization approach.

\section{Limitations and Future Directions}

While our random Fourier feature-based approach demonstrates clear empirical and theoretical advantages, several limitations and promising avenues for future exploration remain. Theoretically, our current analysis is restricted to fully-connected neural networks and deterministic input mappings, without explicitly characterizing input distributions prone to kernel instability. Extensions to more complex neural architectures, such as convolutional and recurrent networks, along with comprehensive analyses of convergence behavior under diverse input distributions, could further validate and enrich our theoretical contributions. Additionally, quantifying how kernel misalignment impacts performance and determining the critical threshold of misalignment beyond which random Fourier feature models become harmful appears as an intriguing area for future research.

Practically, although the pre-processing step introduces negligible computational overhead, selecting optimal parameters (such as the bandwidth or dimensionality of the random Fourier features) still requires empirical tuning. Developing systematic and principled approaches for parameter selection would enhance the method's usability and scalability in real-world scenarios.

\section{Conclusion}
\label{sec: conclude}
We introduce a simple parameter free preprocessing step, random Fourier features, that can be plugged into any tabular deep learning model. Through neural tangent kernel (NTK) analysis we show that this transform (i) bounds the initial kernel spectrum to prevent ill conditioning and (ii) injects a bias term that shortens the effective optimization path. Empirically our approach accelerates training and yields consistent performance gains across diverse benchmarks without adding learnable parameters or significant runtime overhead. These results demonstrate that random Fourier features are a lightweight yet powerful regularizer that significantly accelerates convergence of DL on tabular tasks.

\bibliographystyle{apalike}
\bibliography{references}

\begin{thebibliography}{}

\bibitem[Adlam and Pennington, 2020]{adlam_2020_highdimntk}
Adlam, B. and Pennington, J. (2020).
\newblock The neural tangent kernel in high dimensions: Triple descent and a multi-scale theory of generalization.
\newblock In {\em International Conference on Machine Learning}, pages 74--84. PMLR.

\bibitem[Arik and Pfister, 2021]{arik2021tabnet}
Arik, S.~{\"O}. and Pfister, T. (2021).
\newblock Tabnet: Attentive interpretable tabular learning.
\newblock In {\em Proceedings of the AAAI conference on artificial intelligence}, volume~35, pages 6679--6687.

\bibitem[Atanasov et~al., 2021]{atanasov_2021_alignmentntk}
Atanasov, A., Bordelon, B., and Pehlevan, C. (2021).
\newblock Neural networks as kernel learners: The silent alignment effect.
\newblock {\em arXiv preprint arXiv:2111.00034}.

\bibitem[Bahri et~al., 2021]{bahri2021scarf}
Bahri, D., Jiang, H., Tay, Y., and Metzler, D. (2021).
\newblock Scarf: Self-supervised contrastive learning using random feature corruption.
\newblock {\em arXiv preprint arXiv:2106.15147}.

\bibitem[Baldi et~al., 2014]{baldi2014searching}
Baldi, P., Sadowski, P., and Whiteson, D. (2014).
\newblock Searching for exotic particles in high-energy physics with deep learning.
\newblock {\em Nature communications}, 5(1):4308.

\bibitem[Blackard and Dean, 1999]{blackard1999comparative}
Blackard, J.~A. and Dean, D.~J. (1999).
\newblock Comparative accuracies of artificial neural networks and discriminant analysis in predicting forest cover types from cartographic variables.
\newblock {\em Computers and electronics in agriculture}, 24(3):131--151.

\bibitem[Chen et~al., 2023]{chen2023recontab}
Chen, S., Wu, J., Hovakimyan, N., and Yao, H. (2023).
\newblock Recontab: Regularized contrastive representation learning for tabular data.
\newblock {\em arXiv preprint arXiv:2310.18541}.

\bibitem[Chen et~al., 2020]{chen_2020_generalizedntk}
Chen, Z., Cao, Y., Gu, Q., and Zhang, T. (2020).
\newblock A generalized neural tangent kernel analysis for two-layer neural networks.
\newblock {\em Advances in Neural Information Processing Systems}, 33:13363--13373.

\bibitem[Domingos, 2020]{domingos_2020_sgdntk}
Domingos, P. (2020).
\newblock Every model learned by gradient descent is approximately a kernel machine.
\newblock {\em arXiv preprint arXiv:2012.00152}.

\bibitem[Gage, 1994]{gage1994new}
Gage, P. (1994).
\newblock A new algorithm for data compression.
\newblock {\em The C Users Journal}, 12(2):23--38.

\bibitem[Golikov et~al., 2022]{golikov_2022_ntk_survey}
Golikov, E., Pokonechnyy, E., and Korviakov, V. (2022).
\newblock Neural tangent kernel: A survey.
\newblock {\em arXiv preprint arXiv:2208.13614}.

\bibitem[Gorishniy et~al., 2022]{gorishniy2022embeddings}
Gorishniy, Y., Rubachev, I., and Babenko, A. (2022).
\newblock On embeddings for numerical features in tabular deep learning.
\newblock {\em Advances in Neural Information Processing Systems}, 35:24991--25004.

\bibitem[Gorishniy et~al., 2024]{gorishniy_2024_tabr}
Gorishniy, Y., Rubachev, I., Kartashev, N., Shlenskii, D., Kotelnikov, A., and Babenko, A. (2024).
\newblock Tabr: Tabular deep learning meets nearest neighbors.
\newblock In {\em The Twelfth International Conference on Learning Representations}.

\bibitem[Gorishniy et~al., 2021]{gorishniy2021revisiting}
Gorishniy, Y., Rubachev, I., Khrulkov, V., and Babenko, A. (2021).
\newblock Revisiting deep learning models for tabular data.
\newblock {\em Advances in neural information processing systems}, 34:18932--18943.

\bibitem[Hazimeh et~al., 2020]{hazimeh_2020_treednn}
Hazimeh, H., Ponomareva, N., Mol, P., Tan, Z., and Mazumder, R. (2020).
\newblock The tree ensemble layer: Differentiability meets conditional computation.
\newblock In {\em International Conference on Machine Learning}, pages 4138--4148. PMLR.

\bibitem[Hegselmann et~al., 2023]{hegselmann_2023_tabllm}
Hegselmann, S., Buendia, A., Lang, H., Agrawal, M., Jiang, X., and Sontag, D. (2023).
\newblock Tabllm: Few-shot classification of tabular data with large language models.
\newblock In {\em International Conference on Artificial Intelligence and Statistics}, pages 5549--5581. PMLR.

\bibitem[Holzm{\"u}ller et~al., 2025]{holzmuller_2025_realmlp}
Holzm{\"u}ller, D., Grinsztajn, L., and Steinwart, I. (2025).
\newblock Better by default: Strong pre-tuned mlps and boosted trees on tabular data.
\newblock {\em Advances in Neural Information Processing Systems}, 37:26577--26658.

\bibitem[Huang et~al., 2020]{huang_2020_tabtransformer}
Huang, X., Khetan, A., Cvitkovic, M., and Karnin, Z. (2020).
\newblock Tabtransformer: Tabular data modeling using contextual embeddings.
\newblock {\em arXiv preprint arXiv:2012.06678}.

\bibitem[{IBM}, 2019]{ibm2019telco}
{IBM} (2019).
\newblock Telco customer churn (11.1.3+).
\newblock \url{https://community.ibm.com/community/user/businessanalytics/blogs/steven-macko/2019/07/11/telco-customer-churn-1113}.
\newblock Accessed: 2025-05-14.

\bibitem[{INRIA}, 2021]{inria_california_housing}
{INRIA} (2021).
\newblock California housing dataset — scikit-learn mooc.
\newblock \url{https://inria.github.io/scikit-learn-mooc/python_scripts/datasets_california_housing.html}.
\newblock Accessed: 2025-05-14.

\bibitem[Jacot et~al., 2018]{jacot_2018_ntk}
Jacot, A., Gabriel, F., and Hongler, C. (2018).
\newblock Neural tangent kernel: Convergence and generalization in neural networks.
\newblock {\em Advances in neural information processing systems}, 31.

\bibitem[Jeffares et~al., 2025]{jeffares_2025_telescopingntk}
Jeffares, A., Curth, A., and van~der Schaar, M. (2025).
\newblock Deep learning through a telescoping lens: A simple model provides empirical insights on grokking, gradient boosting \& beyond.
\newblock {\em Advances in Neural Information Processing Systems}, 37:123498--123533.

\bibitem[Klambauer et~al., 2017]{klambauer_2017_snn}
Klambauer, G., Unterthiner, T., Mayr, A., and Hochreiter, S. (2017).
\newblock Self-normalizing neural networks.
\newblock {\em Advances in neural information processing systems}, 30.

\bibitem[Kohavi et~al., 1996]{kohavi1996scaling}
Kohavi, R. et~al. (1996).
\newblock Scaling up the accuracy of naive-bayes classifiers: A decision-tree hybrid.
\newblock In {\em Kdd}, volume~96, pages 202--207.

\bibitem[LeCun et~al., 1989]{lecun1989handwritten}
LeCun, Y., Boser, B., Denker, J., Henderson, D., Howard, R., Hubbard, W., and Jackel, L. (1989).
\newblock Handwritten digit recognition with a back-propagation network.
\newblock {\em Advances in neural information processing systems}, 2.

\bibitem[Liu et~al., 2024]{liu_2024_kan}
Liu, Z., Wang, Y., Vaidya, S., Ruehle, F., Halverson, J., Solja{\v{c}}i{\'c}, M., Hou, T.~Y., and Tegmark, M. (2024).
\newblock Kan: Kolmogorov-arnold networks.
\newblock {\em arXiv preprint arXiv:2404.19756}.

\bibitem[Mondal, 2019]{mondal_santander_2019}
Mondal, S.~K. (2019).
\newblock Santander customer satisfaction.
\newblock \url{https://github.com/SurajKumarMondal/Santander-Customer-Satisfaction}.
\newblock GitHub repository, Accessed: 2025-05-14.

\bibitem[{OpenML}, 2014a]{openmlGesture4538}
{OpenML} (2014a).
\newblock Gesturephasesegmentationprocessed.
\newblock \url{https://www.openml.org/d/4538}.
\newblock OpenML Dataset. ID: 4538, Accessed: 2025-05-14.

\bibitem[{OpenML}, 2014b]{openml_house_16h_574}
{OpenML} (2014b).
\newblock house\_16h dataset.
\newblock \url{https://www.openml.org/d/574}.
\newblock OpenML Dataset. ID: 574, Accessed: 2025-05-14.

\bibitem[Pace and Barry, 1997]{pace1997sparse}
Pace, R.~K. and Barry, R. (1997).
\newblock Sparse spatial autoregressions.
\newblock {\em Statistics \& Probability Letters}, 33(3):291--297.

\bibitem[Popov et~al., 2019]{popov2019neural}
Popov, S., Morozov, S., and Babenko, A. (2019).
\newblock Neural oblivious decision ensembles for deep learning on tabular data.
\newblock {\em arXiv preprint arXiv:1909.06312}.

\bibitem[Qin and Liu, 2013]{qin2013introducing}
Qin, T. and Liu, T.-Y. (2013).
\newblock Introducing letor 4.0 datasets.
\newblock {\em arXiv preprint arXiv:1306.2597}.

\bibitem[Rahimi and Recht, 2007]{rahimi2007random}
Rahimi, A. and Recht, B. (2007).
\newblock Random features for large-scale kernel machines.
\newblock {\em Advances in neural information processing systems}, 20.

\bibitem[Somepalli et~al., 2021]{somepalli2021saint}
Somepalli, G., Goldblum, M., Schwarzschild, A., Bruss, C.~B., and Goldstein, T. (2021).
\newblock Saint: Improved neural networks for tabular data via row attention and contrastive pre-training.
\newblock {\em arXiv preprint arXiv:2106.01342}.

\bibitem[Sotoudeh and Thakur, 2020]{sotoudeh_2020_ann}
Sotoudeh, M. and Thakur, A.~V. (2020).
\newblock Abstract neural networks.
\newblock In {\em Static Analysis: 27th International Symposium, SAS 2020, Virtual Event, November 18--20, 2020, Proceedings 27}, pages 65--88. Springer.

\bibitem[Sutherland and Schneider, 2015]{sutherland2015error}
Sutherland, D.~J. and Schneider, J. (2015).
\newblock On the error of random fourier features.
\newblock {\em arXiv preprint arXiv:1506.02785}.

\bibitem[Vaswani et~al., 2017]{vaswani_2017_attention}
Vaswani, A., Shazeer, N., Parmar, N., Uszkoreit, J., Jones, L., Gomez, A.~N., Kaiser, {\L}., and Polosukhin, I. (2017).
\newblock Attention is all you need.
\newblock {\em Advances in neural information processing systems}, 30.

\bibitem[Wu et~al., 2024]{wu_2024_switchtab}
Wu, J., Chen, S., Zhao, Q., Sergazinov, R., Li, C., Liu, S., Zhao, C., Xie, T., Guo, H., Ji, C., et~al. (2024).
\newblock Switchtab: Switched autoencoders are effective tabular learners.
\newblock In {\em Proceedings of the AAAI Conference on Artificial Intelligence}, volume~38, pages 15924--15933.

\bibitem[Ye et~al., 2024a]{ye_2024_ptarl}
Ye, H., Fan, W., Song, X., Zheng, S., Zhao, H., dan Guo, D., and Chang, Y. (2024a).
\newblock Ptarl: Prototype-based tabular representation learning via space calibration.
\newblock In {\em The Twelfth International Conference on Learning Representations}.

\bibitem[Ye et~al., 2024b]{ye_2024_modernnca}
Ye, H.-J., Yin, H.-H., and Zhan, D.-C. (2024b).
\newblock Modern neighborhood components analysis: A deep tabular baseline two decades later.
\newblock {\em arXiv preprint arXiv:2407.03257}.

\bibitem[Yin et~al., 2020]{yin2020tabert}
Yin, P., Neubig, G., Yih, W.-t., and Riedel, S. (2020).
\newblock Tabert: Pretraining for joint understanding of textual and tabular data.
\newblock {\em arXiv preprint arXiv:2005.08314}.

\bibitem[Yoon et~al., 2020]{yoon2020vime}
Yoon, J., Zhang, Y., Jordon, J., and Van~der Schaar, M. (2020).
\newblock Vime: Extending the success of self-and semi-supervised learning to tabular domain.
\newblock {\em Advances in neural information processing systems}, 33:11033--11043.

\end{thebibliography}

\newpage

\newpage
\appendix

\section{Preliminary facts}

We repeat our notation and assumption set for convenience of the reader.

\paragraph{Notation}
Define the output of the $(k+1)$-th layer of an $m$-layer Multi-layer Perceptron (MLP) as
\begin{equation}
    \mathbf{h}^{k+1}=\sigma_k(W_k \mathbf{h}^{k}) \in \mathbb{R}^{d_{k+1}},
\end{equation}
with the initial input $\mathbf{h}^{0}=\mathbf{x} \in \mathbb{R}^d$. The final output is given by
\begin{equation}\label{equ:mlp}
    \hat{y} = f_\theta(x)=\mathbf{h}^{m}\in\mathbb R,
\end{equation}
with the weight matrices $W_k$ collected in $\theta=\{W_k\}_{k=0}^{m-1}$. 

We denote the Frobenius norm by $\|\cdot\|_F$, the matrix (spectral) 2-norm by $\|\cdot\|_2$, and the vector (Euclidean) 2-norm by $\|\cdot\|$. Furthermore, we denote vectorization by $\vec(\mathbf{A})$.

\setcounter{assumption}{0}
\begin{assumption}
\label{ass:main_assumption}
We make the following standard assumptions:
\begin{enumerate}
    \item The Multi-Layer Perceptron $f_{\bm{\theta}}(\cdot)$ is non-degenerate, i.e., $f_{\bm{\theta}}(\mathbf{x}) \neq \text{const}$ for all $\mathbf{x} \in \mathbb{R}^d$.
    \item Each activation function $\sigma_k$ in every layer $k$ of the neural network is $1$‑Lipschitz and $s$‑homogeneous.
\end{enumerate}
\end{assumption}

\begin{lemma}\label{lemma:homogenuous}

Suppose $f_{\bm{\theta}}$ is the  MLP satisfying Assumption~\ref{ass:main_assumption}. Then, $f_{\bm{\theta}}$ is $s^m$-homogeneous. Further, the gradient with respect to the $k^{th}$-hidden units can be written as:
\begin{equation}
    \nabla_{\mathbf{h}^k}f_{\bm{\theta}}(\alpha \, \mathbf{x}) = \alpha^{s^{m-k}-1} \nabla_{\mathbf{h}^k}f_{\bm{\theta}}(\mathbf{x})
\end{equation}
\end{lemma}
\begin{proof}

    Define the truncated neural network $g^k$ where which takes as an input $\mathbf{h}^k$ and passes it through the layers $k, k+1, \dots, m$. In particular, we have $g^0(\mathbf{x}) = f_{\bm{\theta}}(\mathbf{x})$ and $g^m(\mathbf{h}^m) = \mathbf{h}^m$. 

    We proceed to prove the claim by induction. In the base case, we have $g^m$ which is clearly $1$-homogeneous. Suppose that $g^{k+1}$ is $d^{m-k-1}$-homogeneous. Then we have:
    \begin{equation}
        \begin{split}
            g^{k} (\alpha \, \mathbf{h}^k) &= g^{k+1} (\sigma(W_k \alpha \, \mathbf{h}^k)) \\
            &= g^{k+1} (\alpha^s\, \sigma(W_k \mathbf{h}^k)) \\
            &= (\alpha^s)^{s^{m-k-1}} g^{k} (\mathbf{h}^k)).
        \end{split}
    \end{equation}
    Hence, we have that $g^k$ is $s^{m-k}$-homogeneous, proving the induction.

    Recalling that if $g$ is $s$-homogeneous then $\nabla_{\mathbf{h}}g(\alpha \, \mathbf{h}) = \alpha^{s-1} \nabla_{\mathbf{h}}g(\mathbf{h})$, we have that:
    \begin{equation}
        \nabla_{\mathbf{h}^k}f_{\bm{\theta}}(\alpha \, \mathbf{x}) = \alpha^{s^{m-k}-1} \nabla_{\mathbf{h}^k}f_{\bm{\theta}}(\mathbf{x}).
    \end{equation}    
\end{proof}


\begin{lemma}
\label{lemma:backprop}
Define the forward prefix up to layer $k$ and backward suffix from layer $k$ as:
\begin{equation}
    S_k = \prod_{i=0}^{k-1}\|W_i\|_2, \quad T_k = \prod_{i=k}^{m-1}\|W_i\|_2,
\end{equation}
where $S_0 = 1$ and $T_m = 1$. Then, for any MLP satisfying Assumption \ref{ass:main_assumption}, the following inequalities for the hidden layers hold:
\begin{align}
    \|\mathbf{h}^k\| \leq S_k \|\mathbf{x}\| \qquad
    \|\nabla_{\mathbf{h}^{k}} f_{\bm{\theta} }(\mathbf{x})\| \leq T_{k} 
\end{align}

\end{lemma}

\begin{proof} For both proofs, we proceed by induction on the index of the layer. 

For the first inequality, since $\mathbf{h}^0 = \mathbf{x}$ by definition, the inequality holds. Supposing it holds for $k$, we write:
\begin{equation}
    \begin{split}
        \|\mathbf{h}^{k+1}\| &= \|\sigma(W_k \mathbf{h}^{k}) - \sigma(W_k \mathbf{0})\| \\
    &\stackrel{(i)}{\le} \|W_k \mathbf{h}^{k}\| \\
    &\stackrel{(ii)}{\le} \|W_k \|_2 \|\mathbf{h}^{k}\| \\
    &\stackrel{(iii)}{\le} \|W_k \|_2 S_k \|\mathbf{x}\| = S_{k+1}\|\mathbf{x}\|,    
    \end{split}
\end{equation}
where $(i)$ holds by 1-Lipschitzness of the activation function by Assumption~\ref{ass:main_assumption}, $(ii)$ holds by the properties of the 2-norm and $(iii)$ is the induction hypothesis.

Similarly, for the second inequality, since $f_{\bm{\theta}} = \mathbf{h}^{m}$, the base case holds. Define $D_k = \text{diag}[\sigma^\prime]$ Now supposing the inequality holds for $k+1$
Similarly, for the second inequality, since $f_{\bm{\theta}} = \mathbf{h}^{m}$, the base case holds. Define $D_k = \text{diag}[\{\sigma^\prime(\mathbf{h}^k_i)\}_i]$. Now supposing the inequality holds for $k+1$, we have:
\begin{equation}
    \begin{split}
       \|\nabla_{\mathbf{h}^k} f_{\bm{\theta} }(\mathbf{x})\| &= \|(\nabla_{\mathbf{h}^{k+1}} f_{\bm{\theta} }(\mathbf{x}))^\top D_{k} W_{k}\| \\
       &\leq \|\nabla_{\mathbf{h}^{k+1}}f_{\bm{\theta}}(\mathbf{x})\| \|D_{k}\|_2 \|W_{k}\|_2 \\
       &\stackrel{(i)}{\leq} \|\nabla_{\mathbf{h}^{k+1}}f_{\bm{\theta}}(\mathbf{x})\|\|W_{k}\|_2 \\
       &\stackrel{(ii)}{\leq} T_{k+1} \|W_{k}\|_2  = T_{k},
    \end{split}
\end{equation}
where $(i)$ holds by 1-Lipschitzness of the activation function by Assumption~\ref{ass:main_assumption} and $(ii)$ holds by the induction hypothesis. 
\end{proof}


The following Theorem follows from Lemma 2 and Lemma 3, which establish upper and lower bounds on the norm of the gradient of the neural network.

\begin{theorem}
\label{thm:bounds}    
    Define the vectorized gradient of the MLP with respect to its parameters as $\nabla_\theta f_\theta(\cdot)$. 
    \begin{enumerate}
        \item  There exists a constant $C_1 > 0$ such that
    \begin{equation}
      \|\nabla_\theta f_\theta(\mathbf{x})\| \leq C_1 \, \|\mathbf{x}\|,
    \end{equation}
    where $C_1>0$ depends only on the network weights.
    \item  In addition, there exists a vector $\mathbf{x}$ and a constant $C_2 > 0$ such that for any $\alpha > 1$, we have:
    \begin{equation}
        \|\nabla_\theta f_\theta(\alpha \cdot \mathbf{x})\| \ge C_2 \, \alpha^{s^{m-1}},
    \end{equation}
    where $C_2>0$ does not depend on $\alpha$.
    \end{enumerate}
\end{theorem}

\begin{proof}
We compute the partial derivative with respect to the weight $W_k$ at layer $k$:
\begin{equation}\label{equ:fro_ineq}
\begin{split}
    \Bigl\|\frac{\partial f_\theta(\mathbf{x})}{\partial W_k}\Bigr\|_F &= \Bigl\| \nabla_{\mathbf{h}^{k+1}} f_\theta(\mathbf{x}) \, (\mathbf{h}^k)^\top\Bigr\|_F \\
    &\stackrel{(i)}{=}  \| \nabla_{\mathbf{h}^{k+1}} f_\theta(\mathbf{x})\| \, \| \mathbf{h}^k\|\\
    &\stackrel{(ii)}{\le} T_{k+1} S_k \|\mathbf{x}\|,
\end{split}
\end{equation}
where $(i)$ follows from the definition property of the Frobenius and vector 2-norms, $(ii)$ uses Lemma~\ref{lemma:backprop}. Finally,
\begin{equation}
    \begin{split}
        \| \nabla_\theta \, f_\theta(\mathbf{x})\| &= \sqrt{ \sum_{k=0}^{m-1} \Bigl\| \tfrac{\partial f_\theta(\mathbf{x})}{\partial W_k} \Bigr\|_F^2 } \\
        &\stackrel{(i)}{\le} \sum_{k=0}^{m-1} \Bigl\| \tfrac{\partial f_\theta(\mathbf{x})}{\partial W_k} \Bigr\|_F \\
        &\stackrel{(ii)}{\le} \sum_{k=0}^{m-1}  T_{k+1}  \, S_k \, \|\mathbf{x}\|,
    \end{split}
\end{equation}
where $(i)$ follows from the triangle inequality and $(ii)$ from the bound above. Defining $C_1 = \sum_{k=0}^{m-1}  T_{k+1}  \, S_k$ completes the proof of Part 1.

Next, we are going to prove the lower bound. First, from the definition of $\nabla_\theta f_\theta(\cdot)$, we have
\begin{equation}
\begin{split}
    \| \nabla_\theta f_\theta(\alpha \mathbf{x}) \| &\ge \|\nabla_{\mathbf{h}^1} \, f_\theta(\alpha \, \mathbf{x}) \, \alpha \, \mathbf{x}^\top  \|_F \\&= \| \nabla_{\mathbf{h}^1} \, f_\theta(\alpha \cdot \mathbf{x})\| \, \|\alpha \mathbf{x}\|\\
    &\stackrel{(i)}{=} \| \alpha^{s^{m-1}-1} \, \nabla_{\mathbf{h}^1} \, f_\theta(\mathbf{x})\| \, \| \alpha \, \mathbf{x}\| = \alpha^{s^{m-1}} \,  \|\nabla_{\mathbf{h}^1} \, f_\theta(\mathbf{x})\|  \, \| \mathbf{x}\|,
\end{split}
\end{equation}
where $(i)$ is from Lemma \ref{lemma:homogenuous}. Then, from the non-degenerateness in Assumption \ref{ass:main_assumption}, we can select $\mathbf{x}\neq \mathbf{0}$ such that $\|\nabla_{\mathbf{h}^1} \, f_\theta(\mathbf{x})\| > 0$. Choosing $C_2 = \,\|\nabla_{\mathbf{h}^1} \, f_\theta(\mathbf{x})\| \|\mathbf{x}\|$, which completes the proof of the second part.

\end{proof}

\begin{lemma}[Theorem 1 of \citet{rahimi2007random}]\label{lem:bochner}
Let $k(\delta)$ be any shift‑invariant, real‑valued, positive‑definite kernel on $\mathbb{R}^{d}$ with Fourier transform $p(\omega)$ scaled so that $\int p(\omega)d\omega=1$.  Define $\zeta_{\omega}(\mathbf{x}) := e^{i \omega^{\top}\mathbf{x}}$.  Then Bochner \citep{} theorem guarantees that:
\begin{equation}
\mathbb{E}_{\omega}[\,\zeta_{\omega}(\mathbf{x})\,\overline{\zeta_{\omega}(\mathbf{y})}\,] = k(\mathbf{x}-\mathbf{y}), \qquad\forall\,\mathbf{x},\mathbf{y}\in\mathbb{R}^{d}.    
\end{equation}

Consequently, the random variable $\zeta_{\omega}(\mathbf{x})\overline{\zeta_{\omega}(\mathbf{y})}$ is an unbiased estimator of the kernel value.
\end{lemma}

A real embedding is obtained by the usual trigonometric identity $e^{iz}=\cos z+ i \sin z$. Hence, the inner product $\Phi^\top(\mathbf{x}) \Phi(\mathbf{x})$ is an unbiased estimate of the kernel $k$.

\section{Proof of Theorem 1}
\label{app:proof-thm1}
\setcounter{theorem}{0}

\begin{theorem}[Kernel Boundedness and Bias]
\label{thm:boundedness}
Consider kernels $K_t$ and $\tilde K_t$ associated with a neural network $f_\theta: \mathbf{x} \to \hat{y}$, trained without and with random Fourier features $\Phi(\mathbf{x})$, respectively. Then, the following properties hold:
\begin{enumerate}
    \item As the input feature becomes unbounded, i.e., $\| \mathbf{x}\| \to \infty$, the kernel $K_t$ becomes unbounded, specifically,
    \[
    \sup_{\mathbf{x},\mathbf{x}'} K_t(\mathbf{x},\mathbf{x}') = \infty.
    \]
    \item Under Assumption \ref{ass:main_assumption}, the kernel $\tilde K_t$ admits an upper-bound decomposition:
    \[
    \tilde K_t(\mathbf{x}, \mathbf{x}_i) = \gamma_t(\mathbf{x}, \mathbf{x}_i) + m_t(\mathbf{x}, \mathbf{x}_i),
    \]
    where $\gamma_t$ is a bounded, positive-definite kernel, and $m_t$ is a bounded function approximating a shift-invariant kernel determined by the sampling distribution $p(\omega)$. Specifically,
    \[
    \sup_{\mathbf{x},\mathbf{x}'} \gamma_t(\mathbf{x},\mathbf{x}') < \infty, \quad \text{and} \quad \sup_{\mathbf{x},\mathbf{x}'} m_t(\mathbf{x},\mathbf{x}') < \infty.
    \]
\end{enumerate}
\end{theorem}

\begin{proof} For \textbf{unboundedness}, we apply Theorem 1. Specifically, there exists $\mathbf{x}$ such that:
\begin{equation}
    K_t(\alpha \, \mathbf{x}, \alpha \, \mathbf{x}) = \left\| \nabla_{\boldsymbol{\theta}}f_{\boldsymbol{\theta}_{t-1}}(\alpha \, \mathbf{x}) \right\|_2^2 \ge C_2^2 \alpha^{2s^{m-1}}. 
\end{equation}
Therefore, we have that \(\sup_{\mathbf{x},\mathbf{x}'}K_t(\mathbf{x},\mathbf{x}')=\infty\). 

To show the \textbf{decomposition} of $\tilde K_t$, first define the truncated neural network $g$ which takes as an input $\mathbf{h}^1$ and passes it through the layers $1, 2, \dots, m$. Denote $\bm{\psi}_{t-1} = \{W_k\}_{k=1}^{m-1}$ the parameters of $g$. Then, we write the vectorized gradient as: 
\begin{align}
    \nabla_{\boldsymbol{\theta}} f_{\boldsymbol{\theta}_{t-1}}(\Phi(\mathbf{x})) &= \left[\vect{\left(\frac{\partial f}{\partial W_{k-1}}\right)}, \dots, \vect\left(\frac{\partial f}{\partial W_{1}}\right), \vect\left(\frac{\partial f}{\partial W_{0}}\right)\right] \\
    &= \left [ \nabla_{\bm{\psi}} g_{\bm{\psi}_{t-1}}(\mathbf{h}_1),  \vect\left(\nabla_{\mathbf{h}^1} f_{\bm{\theta}^{t-1}}(\Phi(\mathbf{x})) \Phi(\mathbf{x})^\top\right)\right].
\end{align}
Let's write the vectorized term as:
\begin{equation}
    \vect\left(\nabla_{\mathbf{h}^1} f_{\bm{\theta}^{t-1}}(\Phi(\mathbf{x})) \Phi(\mathbf{x})^\top\right) = \left[\frac{\partial f}{\partial \mathbf{h}^1_1} \Phi(\mathbf{x}), \frac{\partial f}{\partial \mathbf{h}^1_2} \Phi(\mathbf{x}),  \dots, \frac{\partial f}{\partial \mathbf{h}^1_{d_1}} \Phi(\mathbf{x})\right]
\end{equation}
Our modified kernel then becomes:
\begin{equation}
    \tilde K_t(\mathbf{x}, \mathbf{x}^\prime) = \underbrace{\left(\nabla_{\bm{\psi}} g_{\bm{\psi}_{t-1}}(\mathbf{h}^1)\right)^\top \nabla_{\bm{\psi}} g_{\bm{\psi}_{t-1}}((\mathbf{h}^1)^\prime)}_{\gamma_t(\mathbf{x, \mathbf{x}^\prime})} + \underbrace{\left(\nabla_{\mathbf{h}^1}f_{\bm{\theta}_{t-1}}(\mathbf{x})\right)^\top \nabla_{\mathbf{h}^1}f_{\bm{\theta}_{t-1}}(\mathbf{x}^\prime) \Phi(\mathbf{x})^\top \Phi(\mathbf{x}^\prime)}_{m_t(\mathbf{x}, \mathbf{x}^\prime)}
\end{equation}
From here, we observe that:
\begin{equation}
\begin{split}
    \|\mathbf{h}^1\| &= \left\| \sigma( W_0 \Phi(\mathbf{x}))\right\| \\
    &\stackrel{(i)}{\leq} \|W_0\|_2 \|\Phi(\mathbf{x})\| \\
    & \stackrel{(ii)}{\leq} \sqrt{2} \|W_0\|_2,
\end{split}
\end{equation}
where $(i)$ follows by 1-Lipschitzness of the activation function by Assumption~\ref{ass:main_assumption} and $(ii)$ follows by definition . Hence, by Theorem~\ref{thm:bounds}, $ \|\nabla_{\bm{\psi}} g_{\bm{\psi}_{t-1}}(\mathbf{h}^1)\|$ is unifromly bounded above by a constant $\sqrt{B} < \infty$ that only depends on the network weights. Therefore, by Cauchy–Bunyakovsky–Schwarz, we have that 
\begin{equation}
    |\gamma_t(\mathbf{x, \mathbf{x}^\prime})| \leq \|\nabla_{\bm{\psi}} g_{\bm{\psi}_{t-1}}(\mathbf{h}^1)\| \|\nabla_{\bm{\psi}} g_{\bm{\psi}_{t-1}}((\mathbf{h}^1)^\prime)\| \leq B.
\end{equation}
Invoking Lemma~\ref{lem:bochner}, we get $m_t$:
\begin{align}
    m_t(\mathbf{x}, \mathbf{x}^\prime) &= \left(\nabla_{\mathbf{h}^1}f_{\bm{\theta}_{t-1}}(\mathbf{x})\right)^\top \nabla_{\mathbf{h}^1}f_{\bm{\theta}_{t-1}}(\mathbf{x}^\prime) \Phi(\mathbf{x})^\top \Phi(\mathbf{x}^\prime) \\
    &\approx B_{t-1}  k(\mathbf{x}, \mathbf{x}^\prime),
\end{align}
where $B_{t-1}$ is finite depending on the network weights by Lemma~\ref{lemma:backprop} and $k$ depends on the sampling distribution $p$ by Lemma 2. It also follows that $\sup_{\mathbf{x},\mathbf{x}'}m_t(\mathbf{x},\mathbf{x}')<\infty$. 

\end{proof}

\end{document}